\documentclass{article}

\PassOptionsToPackage{numbers,sort&compress}{natbib}
\usepackage[preprint]{neurips_2026}
\makeatletter
\renewcommand{\@noticestring}{Preprint. Accepted at the 9th Workshop on
Machine Learning and the Physical Sciences (ML4PS 2026).}
\makeatother

\usepackage[utf8]{inputenc}
\usepackage[T1]{fontenc}
\usepackage{hyperref}
\hypersetup{hidelinks}
\usepackage{url}
\usepackage{booktabs}
\usepackage{amsmath,amssymb,amsthm}
\usepackage{microtype}
\usepackage{graphicx}

\newtheorem{theorem}{Theorem}

\theoremstyle{remark}

\newcommand{\R}{\mathbb{R}}
\newcommand{\Od}{O(d)}

\newcommand{\ip}[2]{\langle #1,\, #2\rangle}
\newcommand{\norm}[1]{\lVert #1\rVert}
\DeclareMathOperator{\tr}{tr}

\title{Attention as Lattice Gauge Theory:\\
Relations Interact Only Through Wilson Loops}

\author{%
  Aakash Sarkar\\
  University of California, San Francisco\\
  \texttt{aakash.sarkar@ucsf.edu}
  \And
  Ananya Mukherjee\\
  Brigham and Women's Hospital\\
  \texttt{amukherjee8@bwh.harvard.edu}
  \And
  Motoshige Sato\\
  University of California, San Francisco\\
  \texttt{motoshige.sato@ucsf.edu}
}

\begin{document}

\maketitle

\begin{abstract}
Relations compose through the entity they share---Alice is Bob's
mother and Bob is Carol's father, so Alice is Carol's
grandmother---and the order matters, since a father's mother is not
a mother's father; verifying a composed relation means closing the
loop. We show that if each entity's coordinate frame may
be rotated independently (per-entity frame freedom~\cite{p1}), three
things follow. In the pure-gauge sector, relations that share no
cycle cannot interact at all.
The only gauge-invariant score of a triangle is a class function of
its holonomy $H{=}T_{ij}T_{jk}T_{ki}$---a Wilson loop. And the only
equivariant update of an edge is $U{=}\Psi(H)\,T_{ij}$ with $\Psi$
conjugation-equivariant. Edge attention by Wilson loops is therefore
not one design among many but the unique form the symmetry permits.
Three communities hold pieces of this form: the commuting-operator
embeddings of knowledge graphs~\cite{sun2019,trouillon2016} are its
abelian (commuting) sector; AlphaFold's triangle update~\cite{jumper2021} and
the Edge Transformer~\cite{bergen2021} have its shape---an edge
updated through the third vertex---without its constraints; and
lattice QCD builds it equivariantly on
lattices~\cite{kanwar2020,nagai2025}. We instantiate the equivariant
sub-family on general relational graphs as \emph{holonomy
attention}. With no training and composition alone, it infers
unobserved $S_4$ relations on sparse graphs at $0.88$ with a fifth
of the edges corrupted, at or slightly below specialist
synchronization algorithms, where every commutative architecture is
capped at $0.083$; its holonomy score flags corrupted triangles at
AUC $0.99$ where commuting scores are capped at $0.76$ by a second
selection rule, and it separates clusters without labels. \end{abstract}

\section{Frame freedom and the transport sector}
\label{sec:setting}

In a lattice gauge theory, matter fields live at sites and gauge
fields on links; comparing matter across sites requires parallel
transport, and the gauge field interacts only around plaquettes, the
elementary closed loops~\cite{wilson1974}. We show this structure is
\emph{forced} in attention: per-entity frame freedom requires
transported scores between entities and plaquette-mediated
interaction between relation operators. The dictionary is direct: sites map to entities, link variables
to relation transports $T_{ij}$, plaquette holonomies to Wilson loops
of relation triangles, flat connections to standard
attention~\cite{p1}, and the gauge self-coupling to the forced form
proved here.

\begin{figure}[t]
\centering
\includegraphics[width=\textwidth]{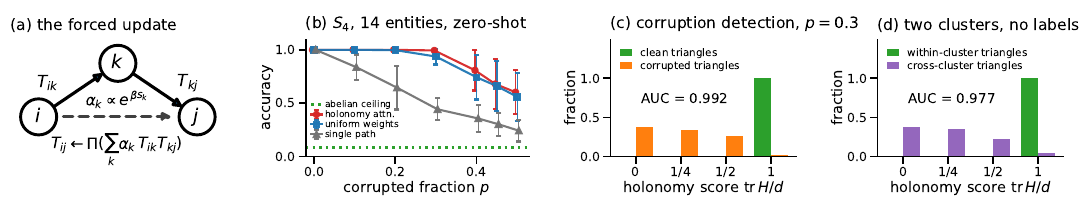}
\caption{Holonomy attention, zero-shot. (a)~The forced update.
(b)~Accuracy on unobserved $S_4$ relations versus corrupted fraction
(14 entities, degree~8, four rounds; mean and sd, $10$ instances).
(c)~Holonomy score $\tr H/d$ of observed triangles at $p{=}0.3$
(discrete: permutation traces are integers). (d)~The same score
separates within- from cross-cluster triangles with no labels.}
\label{fig:gluon}
\end{figure}
 \textbf{Holonomy and frame freedom.} Transporting a vector around
a closed path multiplies it by the \emph{holonomy}, the net rotation
the loop accumulates; its trace, the \emph{Wilson loop}, is a number
independent of the coordinate system at every vertex. Entities introduced on the fly---a name in a story, a variable in a
program---have no intrinsic coordinate system: rotating their
embedding alone does not change the task. A
recent study~\cite{p1} proves that under this $\Od^{\,n}$ symmetry
any admissible attention score depends on the entities only through
$x_i$ and $T_{ij}x_j$ (for isometries, through
$\ip{x_i}{T_{ij}x_j}$), and that because any equivariant map is
rank one, the transport cannot be computed from the two entities and
must be supplied by relation data outside the pair. \emph{The symmetry is the task's, not a network's}: a fresh entity's
frame carries no information, so any predictor must be invariant to
it; a shared-basis model meets that only by learning transports.
\textbf{The flat limit.} When every loop's holonomy is trivial, a
shared frame exists and the forced score reduces to plain dot-product
attention; a standard head realizes exactly this form with a constant
per-head transport composed only through depth, so standard attention
is the flat-connection limit of a richer theory~\cite{p1}.

\textbf{Prior work.} On lattices the form is known: gauge-equivariant
flows, smearing, Wilson-loop attention, and gauge-covariant message
passing all realize conjugation-equivariant plaquette
updates~\cite{kanwar2020,boyda2021,nagai2021,nagai2025,rayat2026},
and such layers are known to suffice~\cite{favoni2022,luo2021}. In
synchronization, holonomy classifies edge potentials~\cite{gao2019}
and cycle inconsistency is the standard
diagnostic~\cite{zach2010,lerman2022}. In ML more broadly, gauge
structure appears through frames on meshes~\cite{cohen2019,he2021},
operator readings of attention~\cite{lin2026}, orthogonal state
models~\cite{sung2026}, and sheaf
networks~\cite{hansen2020,bodnar2022,barbero2022}, none with
interaction between edge maps. What we add is the derivation from
per-entity frame freedom on general relational graphs---the
closed-form $\Od$ statement, where the reflections matter---its
reading across three literatures, and the measurements below. Each edge carries
$T_{ij}{\in}\Od$ with $T_{ji}{=}T_{ij}^{-1}$, and a gauge
transformation (an independent change of coordinates at each entity)
acts by $T_{ij}{\mapsto}g_iT_{ij}g_j^{-1}$.


\section{Forced plaquettes between relations}
\label{sec:forced}

Everything below follows from one move: gauge-fix along a spanning
tree, and see what freedom is left. The sketches here are complete enough
to check; Appendix~B has the details.

\textbf{No cycle, no interaction.} If a set of edges contains no
cycle, gauge-fixing along its spanning tree sends every transport
to $I$. Any function of the transports invariant under $\Od^{V}$
(one rotation per vertex) is therefore constant. An equivariant
update of an edge must, at the identity, commute with the residual
diagonal $\Od$, so by Schur it is a scalar, $\pm I$; undoing the
gauge gives $U{=}\pm T_{ij}$. In the pure-gauge sector the edge cannot be moved by the others
(with entity vectors, $x_ix_j^{\top}$ is an equivariant edge term
with no cycle). Two edges
sharing a vertex but no closing edge form a tree, so no ``global
edge attention'' respects the symmetry: relation operators interact
only around cycles. What cycle-mediated interaction looks like is then the question,
and the answer is unique.

\textbf{Plaquette invariants are Wilson loops.} Let
$(T_{ij},T_{jk},T_{ki})$ be the transports of a triangle and
$H{=}T_{ij}T_{jk}T_{ki}$ its holonomy. Setting $g_i{=}I$,
$g_j{=}T_{ij}$, $g_k{=}T_{ij}T_{jk}$ sends the triple to $(I,I,H)$,
and the only freedom left is the diagonal, $H{\mapsto}gHg^{-1}$. The
orbits are therefore the conjugacy classes of $H$, and the
continuous invariants are exactly the \emph{class functions} of
$H$: functions with $f(gHg^{-1}){=}f(H)$, which see the eigenvalues
and not the basis, generated by $\tr H^r$,
$r{\le}d$~\cite{procesi1976}. The reflections are load-bearing---in
$SO(2)$ the rotations $R(\pm\theta)$ share a characteristic
polynomial but are not conjugate---and the same argument holds for
a cycle of any length $m$; for several cycles the invariants are
functions of $(H_1,\ldots,H_b)$ under simultaneous conjugation. On a
graph this is the classification of~\cite{gao2019}; on a lattice it
is loop reconstruction~\cite{giles1981}.

\begin{theorem}[Edge-level forced form]
\label{thm:update}
Every continuous equivariant update of $T_{ij}$ from a cycle
through it, with holonomy $H$ based at $i$, is
$U{=}\Psi(H)\,T_{ij}$ with $\Psi(gHg^{-1}){=}g\,\Psi(H)\,g^{-1}$,
and $\Psi$ in the closure of $\{\sum_r c_r(H)\,H^r\}$,
$c_r$~class functions, intersected with~$\Od$.
\end{theorem}
\noindent\emph{Proof intuition.} In the gauge $(I,I,H)$ set
$\Psi(H){:=}U(I,I,H)$. Equivariance under the diagonal gives
$\Psi(gHg^{-1}){=}g\Psi(H)g^{-1}$, and returning to the original
gauge supplies the prefactor $T_{ij}$. On $\Od$, $H^\top{=}H^{-1}$,
so words in $H$ and $H^\top$ reduce to integer powers
(the polynomial case is~\cite{procesi1976}); continuous equivariant
maps are limits of polynomial ones. The same gauge fixing gives, for
several cycles through the edge, $U{=}\Psi(H_1,\ldots,H_b)\,T_{ij}$
under simultaneous conjugation---with a candidate $C_1$ in place of
$T_{ij}$ for an absent edge---which is the case run below.

\textbf{Instances.} Setting $\Psi{=}H^{-1}$ gives
$U{=}T_{ik}T_{kj}$: composition through the third vertex. With
several third vertices and class-function scores
$s_k{=}\phi(\tr H_k,\ldots)$, the forced aggregate is
$U{=}\Pi\!\big(\sum_k\alpha_k\,T_{ik}T_{kj}\big)$ with $\Pi$ the
polar projection to $\Od$ (the nearest orthogonal matrix). Since
$\Pi$ is equivariant wherever the sum is invertible, the aggregate stays in the forced family, with
$\Psi$ a function of the several holonomies. Commuting operators are
simultaneously diagonalizable, so their composition is a
coordinate-wise product: the abelian sector is the commuting-operator
knowledge-graph embeddings---RotatE, ComplEx,
DistMult~\cite{sun2019,trouillon2016,yang2015}---whose commutative
composition DihEdral~\cite{xu2019} escapes with a finite non-abelian
group at the node level; none has an edge-level update. In this sector $H{=}e^{\Phi}$, $\Phi{=}\sum A$, so edges couple only
linearly, through the synchronization
Laplacian~\cite{singer2011,bandeira2013}, as photons do not interact
with each other; the Wilson action's commutator term (Appendix~B) is
the self-interaction that makes the non-abelian sector gluon-like,
and this paper's subject.

The Edge Transformer's and AlphaFold's triangle updates have the
form's shape---an edge updated by a weighted sum over third vertices
of a bilinear message from the other two edges---but none of its
constraints: states are not transports, the bilinear map is
unconstrained (with $d^3$ channels it contains matrix
multiplication), and scores are not invariant. They are the unconstrained superset, the forced form its equivariant
sub-family. A design rule follows: one cycle through an edge has the same
holonomy whichever of its edges is wrong, so one loop detects but
cannot localize; two triangles do, via the square $\tr(H_kH_l^{-1})$
that avoids the edge---score a candidate only by loops that avoid it,
as the Edge Transformer's logit already does.

The update is APE smearing~\cite{albanese1987} with two-link staples
(the paths around a link) and Wilson-loop
weights~\cite{nagai2021,nagai2025}, known on lattices since 1987;
what we add is the proof that nothing else exists on general graphs,
the reading that unifies three literatures, and two selection rules
that the form predicts and the experiments confirm.

\section{The non-abelian sector: holonomy attention}
\label{sec:gluon}

We instantiate the non-abelian sector as \emph{holonomy
attention}: an edge-level attention mechanism whose scores are
Wilson loops and whose messages are compositions through shared
vertices, projected back to~$\Od$ (Figure~\ref{fig:gluon}a). Load the group's faithful
permutation representation ($d{=}4$) as a library, run four rounds
of $T_{ij}\!\leftarrow\!\Pi\!\big(\sum_k\alpha_k\,T_{ik}T_{kj}\big)$
with $\alpha_k{\propto}\exp(\beta s_k)$, $s_k$ the mean over $l{\neq}k$
of $\tr(C_kC_l^{-1})/d$, $\beta{=}8$ (a five-fold sweep moves the
routing gain by $\le3$ points), and snap to the nearest group
element in Frobenius norm.

Observed edges are held fixed; an absent edge is scored by the
mutual consistency $\tr(C_kC_l^{-1})$ of its candidates
$C_k{=}T_{ik}T_{kj}$, and $C_kC_l^{-1}{=}T_{ik}T_{kj}T_{jl}T_{li}$
is the holonomy of the square $i{\to}k{\to}j{\to}l{\to}i$, a loop
that avoids the edge, as the design rule requires. Scoring an edge
instead by a triangle that contains it is self-confirming and costs
$5$--$12$ points (Table~\ref{tab:results}, ``literal''). No
parameters are learned; code is released with the paper.

We test it on group synchronization---recovering frames from noisy
pairwise relations---with outliers: fourteen
entities carry hidden frames in $S_4$, the $24$ permutations of
four objects; a random set
of pairs of fixed size, resampled until connected, is observed;
each observed relation is replaced with probability $p$ by a
different uniform random element, with $T_{ji}{=}T_{ij}^{-1}$
kept; the task is the
relation between an unobserved pair, a standard problem in cryo-EM
and sensor calibration~\cite{singer2011}. Holonomy attention is not a better synchronizer: Singer's
eigenvector method~\cite{singer2011} ties it and edges it by $0$--$5$
points with one reweighting pass, and cycle-edge message
passing~\cite{lerman2022} is ahead on a matched $n{=}12$ pilot by $3$, $3$, $0$, $11$ points at $p{=}0.1$,
$0.2$, $0.3$, $0.45$ (standard error $\approx0.09$). The numbers
test the theory, in three results (Table~\ref{tab:results}).

\begin{table}[t]
\centering\footnotesize
\setlength{\tabcolsep}{3.4pt}
\begin{tabular}{@{}llcccccc@{}}
\toprule
\textbf{Protocol} ($S_4$) & $p$ & \textbf{single path} & \textbf{uniform} & \textbf{holonomy attn.} & \textbf{literal} & \textbf{spectral} & \textbf{spectral+1} \\
\midrule
sparse, deg.~5 & $0.2$ & $0.625$ & $0.872$ & $0.886$ & $0.753$ & $0.897$ & $0.912$ \\
dense, deg.~8 & $0.2$ & $0.645$ & $0.982$ & $0.988$ & $0.935$ & $0.987$ & $0.990$ \\
dense, deg.~8 & $0.3$ & $0.488$ & $0.912$ & $0.937$ & $0.820$ & $0.933$ & $0.949$ \\
dense, deg.~8 & $0.45$ & $0.309$ & $0.623$ & $0.700$ & $0.559$ & $0.678$ & $0.753$ \\
\bottomrule
\end{tabular}
\vspace{2pt}
\caption{Accuracy on unobserved pairs, four rounds, $d{=}4$, $1000$ instances per cell, released harness, seed 2026 (95\% CI $\le\pm0.014$); paired holonomy$-$uniform: $+0.014\pm0.004$, $+0.006\pm0.002$, $+0.026\pm0.004$, $+0.078\pm0.006$. ``Literal'' scores an edge by a triangle containing it; abelian ceiling $0.083$. Figure~\ref{fig:gluon}b--d: earlier implementation, $10$--$15$ instances.}
\label{tab:results}
\end{table}

\textbf{Composition is the capability.} On sparse $S_4$ graphs at
$p{=}0.2$ (Table~\ref{tab:results}) holonomy attention reaches
$0.89$ where the single path reaches $0.63$, and composition through
cycles is $26$ to $45$ points above the single path on every
row; any architecture confined to the abelian sector---the
commuting-operator embeddings---is capped at
$|S_4^{\mathrm{ab}}|/|S_4|{=}0.083$, the fraction of a group that
commuting operators distinguish (the kernel argument of~\cite{p1}),
and a character-level composition measures exactly that ($0.083$).

\textbf{Routing is the refinement.} On the dense graphs of
Figure~\ref{fig:gluon}b holonomy scores and uniform weights coincide
at low corruption; the uniform mean bends where clean two-hop paths
become a minority, $(1{-}p)^2{=}\tfrac12$ at $p{\approx}0.29$, and
the paired gain opens there, $+0.026\pm0.004$ at $p{=}0.3$ and
$+0.078\pm0.006$ at $0.45$; iterated reweighting does better still.

\textbf{The holonomy score is the invariant statistic.} Cycle
inconsistency is the classical corruption
statistic~\cite{zach2010,lerman2022}; the forced form says it is the
\emph{only} gauge-invariant one, and it obeys a second selection
rule. $\tr H/d$ separates corrupted from clean triangles with area
under the ROC curve (AUC) $0.99$: clean triangles score exactly $1$,
one corrupted edge never closes a loop, so the only ties are
triangles with two or more cancelling corruptions ($1.4\%$ in closed
form at $p{=}0.3$), and the AUC sits at its ceiling $0.993$. A commuting score sees only
$G^{\mathrm{ab}}{=}G/[G,G]$, so with one corrupted edge its AUC is capped at
$\tfrac12\big(1+(|G|-|[G,G]|)/(|G|-1)\big)$: $0.80$ for $S_3$, $0.93$
for $D_4$ and $Q_8$, $0.86$ for $A_4$, $0.76$ for $S_4$ (measured
$0.759$ against $0.993$), and chance for $A_5$---the commutator
subgroup of~\cite{p1} capping detection as it caps composition. On a two-cluster graph with random cross-cluster
relations it separates within- from cross-cluster triangles at AUC
$0.98$ (Figure~\ref{fig:gluon}d), with no group labels: unsupervised community detection
from non-abelian frustration; structural balance in signed
networks~\cite{cartwright1956} is the $\det H$ sign product; $\tr H^r$
carries the loop-level content.

\section{Discussion}
\label{sec:disc}

The node-level and edge-level forced forms are two halves of one
architecture. The node step---entity vectors updated by transported messages,
scored by $\ip{x_i}{T_{ij}x_j}$---is the matter propagator: content
flows through a fixed gauge field, the operator resolvent
$(I{-}\gamma\mathcal T)^{-1}$ closing all paths at once~\cite{p1}, the
gauge lift of the random-walk resolvent that word embeddings on
graphs factorize~\cite{qiu2018,levy2014}. The operators never touch each other. The edge
step proved here is the gauge propagator: relation operators
compose through shared vertices, scored by Wilson loops, and refine
each other---including for pairs never observed. A full layer
alternates the two: matter propagates in a background field, the
field updates itself, matter propagates again in the refined field.
The resolvent after holonomy attention is exact exactly when the
effective connection is flat on the attention support, which is
what the routing works toward. 

The next step is making the gauge field dynamical: operators
generated from context rather than retrieved from a library. The
rank-one result of~\cite{p1} already constrains the generator---the
transport cannot come from the two entities, so it must come from
relation data the context supplies outside the pair---and the
forced form constrains how generated operators interact: through
plaquettes, by Wilson-loop scores, in the form $\Psi(H)\,T_{ij}$
and no other. The same reading predicts something about a model
everyone has used: AlphaFold's pair stack is learned flattening,
its triangle updates driving pairwise geometry toward zero plaquette
holonomy, so the pair representation's triangle inconsistency should
fall with depth and recycling. The non-abelian sector belongs where order is physically attested:
multi-plane lensing, where image rotation is a non-commutative
Wilson line of shears~\cite{schneider1992,krause2010} that commutative
emulators miss ($0.51$ against $0.92$; Appendix~B); cryo-EM
orientation recovery; kinematic chains. Relations compose through the entity they share,
and the composition is forced; what remains is learning the operators
themselves. 

\bibliographystyle{unsrtnat}

\appendix
\newpage
\section{Forced Form and Flat Limit (from~\cite{p1})}

\paragraph{Forced Form.}
Let entities $i,j$ carry $x_i,x_j{\in}\R^d$ with independent
frames $x_i{\mapsto}g_ix_i$, $x_j{\mapsto}g_jx_j$, and let a
transport $T_{ij}$ transform as
$T_{ij}{\mapsto}g_iT_{ij}g_j^{-1}$. Any continuous invariant
score depends on the entities only through $x_i$ and $T_{ij}x_j$;
for isometries, $s{=}\phi(\norm{x_i},\norm{x_j},\ip{x_i}{T_{ij}x_j})$.
Any equivariant message lies in $\mathrm{span}\{x_i,T_{ij}x_j\}$.
No fixed bilinear $x_i^\top Mx_j$ is invariant ($g_j{=}I$,
$g_i{=}{-}I$ forces $M{=}0$); no equivariant transport
$T(x_i,x_j)$ is orthogonal (it is
$\phi\,x_ix_j^\top$, rank one).

\paragraph{Flat Limit.}
On a connected graph, trivial holonomy $\Leftrightarrow$
$T_{ij}{=}h_ih_j^{-1}$ for frames~$h$ $\Leftrightarrow$ a gauge
with every $T{=}I$. In that gauge the forced score is
$\phi(\norm{x_i},\norm{x_j},\ip{x_i}{x_j})$: attention in one
shared frame. A standard head has transport $W_Q^\top W_K$ (score)
and $W_OW_V$ (value), constant, covariant only under the global
$\Od$, composed only through depth.

\section{Proof of Theorem~\ref{thm:update}, the cubic
coefficient, and the lensing cell}
\paragraph{Lensing cell.} Twelve lens planes with i.i.d.\ convergence
$\kappa_i\sim\mathcal N(0,0.05^2)$ and shear
$\gamma_{1,2}\sim\mathcal N(0,0.05^2)$ give Jacobians
$A_i{=}(1{-}\kappa_i)I{-}\Gamma_i$ and total distortion
$M{=}A_{12}\cdots A_1$; the label is $\mathrm{sign}\,\omega$,
$\omega{=}\tfrac12(M_{21}{-}M_{12})$. An additive emulator (sum of
per-plane features), an elementwise one (coordinate-wise products)
and a $2{\times}2$ matrix composition are each trained on $10^4$
sequences and tested on $10^3$, three seeds: $0.507/0.515/0.510$,
$0.509/0.515/0.518$, and $0.923/0.920/0.922$. Reversing the planes
maps $M\mapsto M^{\top}$ and $\omega\mapsto-\omega$ exactly, so any
order-invariant predictor sits at $\tfrac12$.

\emph{Proof of Theorem~\ref{thm:update}.} In the gauge
$(I,I,H)$ of \S\ref{sec:forced}, define $\Psi(H){:=}U(I,I,H)$.
The residual gauge freedom is the diagonal~$g$, under which
$(I,I,H)\mapsto(I,I,gHg^{-1})$ and, by equivariance,
$U\mapsto gUg^{-1}$; hence $\Psi(gHg^{-1}){=}g\Psi(H)g^{-1}$.
Undoing the gauge transformation ($g_i{=}I$, $g_j{=}T_{ij}$) gives
$U(T){=}g_i^{-1}\Psi(H)g_j{=}\Psi(H)T_{ij}$. For the form of
$\Psi$: on $\Od$, $H^\top{=}H^{-1}$,
so equivariant polynomial maps $M_d(\R){\to}M_d(\R)$ spanned by
words in $H,H^\top$ reduce to integer powers of~$H$ (the polynomial
classification is~\cite{procesi1976}); continuous maps are limits of
polynomial ones by Stone--Weierstrass on the compact group. The cubic coefficient
$-\tfrac12$ follows analytically: with $X{=}\log H{=}\Phi{+}C{+}O(A^3)$,
skew-symmetry gives $\tr X{=}\tr X^3{=}0$, so
$\tr(I{-}H){=}\tfrac12\|\Phi\|_F^2{-}\tr(\Phi C){+}O(A^4)$ and the
three cyclic commutator traces, two entering with opposite signs,
collapse to the single term; numerically the ratio of the residual
to $-\tfrac12\,\tr(A_{ki}[A_{ij},A_{jk}])$ converges to $1.000$ as
$\|A\|{\to}0$.

\end{document}